\documentclass[sn-mathphys,Numbered]{sn-jnl}

\usepackage{graphicx}%
\usepackage{multirow}%
\usepackage{amsmath,amssymb,amsfonts}%
\usepackage{amsthm}%
\usepackage{mathrsfs}%
\usepackage[title]{appendix}%
\usepackage[dvipsnames]{xcolor}%
\usepackage{textcomp}%
\usepackage{manyfoot}%
\usepackage{booktabs}%
\usepackage{algorithm}%
\usepackage{algorithmicx}%
\usepackage{algpseudocode}%
\usepackage{listings}%
\usepackage{svg}

\usepackage{mysty}

\title[]{Recovery Guarantees of Unsupervised Neural Networks for Inverse Problems trained with Gradient Descent}

\author*[]{\fnm{Nathan} \sur{Buskulic}}\email{nathan.buskulic@unicaen.fr}
\author[]{\fnm{Jalal} \sur{Fadili}}\email{Jalal.Fadili@ensicaen.fr}
\author[]{\fnm{Yvain} \sur{Qu\'eau}}\email{yvain.queau@ensicaen.fr}

\affil[]{\orgdiv{Greyc}, \orgname{Normandie Univ., UNICAEN, ENSICAEN, CNRS}, \orgaddress{\street{6 Boulevard Maréchal Juin}, \city{Caen}, \postcode{14000}, \country{France}}}

\begin{document}

\abstract{
Advanced machine learning methods, and more prominently neural networks, have become  standard to solve inverse problems over the last years. However, the theoretical recovery guarantees of such methods are still scarce and difficult to achieve. Only recently did unsupervised methods such as Deep Image Prior (DIP) get equipped with convergence and recovery guarantees for generic loss functions when trained through gradient flow with an appropriate initialization. In this paper, we extend these results by proving that these guarantees hold true when using gradient descent with an appropriately chosen step-size/learning rate. We also show that the discretization only affects the overparametrization bound for a two-layer DIP network by a constant and thus that the different guarantees found for the gradient flow will hold for gradient descent.
}

\keywords{Inverse problems, Deep Image/Inverse Prior, Overparametrization, Gradient descent, Unsupervised learning}

\maketitle

\section{Introduction}
{\noindent \bf Problem statement.~}
In finite dimension, inverse problems are understood as the task of reliably recovering a vector $\xvc$ in a finite-dimensional vector space (throughout $\R^n$) from its indirect and noisy measurements
\begin{align}\label{eq:prob_inv}
    \yv = \fop\xvc + \veps,
\end{align}
where $\yv\in\R^m$ is the observation vector, $\fop: \R^n \to \R^m$ is a forward operator that we assume linear and $\veps$ is an additive noise. We will use the shorthand notation $\yvc \eqdef \fop\xvc$ to denote the noiseless observation vector.

Due to the large number of scientific and engineering fields where inverse problems appear, it is natural that in recent years sophisticated machine learning algorithms, including those based on (deep) neural networks, were developed to solve them. These methods have shown promising results; for space limitation, we refer to the reviews \cite{arridge_solving_2019,ongie_deep_2020}. Many of these approaches are based on the idea of optimizing a generator network $\mathbf{g}: (\uv,\thetav) \in \R^d\times \R^p \mapsto \xv \in \R^n$, equipped with an activation function $\phi$, to transform an input (latent) variable $\uv\in\R^d$ into a vector $\xv$ as close as possible to the sought-after vector. The optimization/training takes the form of a (possibly stochastic) gradient descent on the parameters $\thetav$ of the network to minimize a loss function $\lossy:\funspacedef{\R^m}{\R_+}, \vv \mapsto \lossy(\vv)$ intended to capture the forward model \eqref{eq:prob_inv} by measuring the discrepancy between the observation $\yv$ and an estimated observation $\vv=\fop\gv(\uv,\thetav)$ generated by the network with parameters/weights $\thetav$.

{\noindent \bf Literature overview.~}
Several theoretical works emerged recently to study the optimization trajectory of overparametrized networks~\cite{bartlett_deep_2021,fang_mathematical_2021}. While the first attempts used unrealistic assumptions, such as the strong convexity of the loss when composed with the network, the next attempts were based on gradient dominated inequalities. This allowed to prove, for networks trained to minimize the mean square error (MSE), an exponential convergence rate to an optimal solution of either gradient flow~\cite{chizat_lazy_2019,robin2022convergence}, or gradient descent~\cite{du_gradient_2019,arora_fine-grained_2019, oymak_overparameterized_2019, oymak_toward_2020,liu2022loss}). These results were generalized in~\cite{buskulic2023convergenceJournal} to the inverse problem setting, providing both convergence and recovery guarantees. Furthermore, the latter work alleviates the restriction of using the MSE as the loss function and gives results for generic loss functions obeying the Kurdyka-\L ojasiewicz inequality (e.g., any semi-algebraic function or even definable on an o-minimal structure)~\cite{loj1,loj3,kurdyka_gradients_1998}. However, these guarantees were only provided for the gradient flow and it is yet unknown how forward Euler discretization of this follows, giving rise to to gradient descent, affects them.


{\noindent \bf Contributions.~}
Gradient descent iteration to learn the network parameters reads
\begin{align}
\thetavtauplus = \thetavtau - \gamma \nabla_{\thetav}\lossy(\fop\gv(\uv,\thetavtau)) ,
\label{eq:grad_descent}
\end{align}
with $\gamma$ the (fixed) descent-size/learning rate. The goal of this work is to analyze under which conditions the iterates \eqref{eq:grad_descent} converge, whether they converge to a zero-loss solution, and what can be said about the recovery guarantees of $\xvc$. 

In Section~\ref{sec:guarantees}, we first show that neural networks trained through gradient descent \eqref{eq:grad_descent} can benefit from convergence and recovery guarantees for general loss functions verifying the Kurdyka-\L ojasiewicz (KL) property. That is, we prove that under a proper initialization and a well-chosen step-size $\gamma$, the network will converge to a zero-loss solution at a rate dependent on the desingularizing function of the KL property of the loss. We also provide a bound on the recovery error of the original vector $\xvc$ which requires a restricted injectivity condition to hold, and we emphasize the trade-off between this condition and the expressivity of the trained network. Then, we give a bound on the overparametrization necessary for a two-layer DIP~\cite{ulyanov_deep_2020} network to benefit from all these guarantees with high probability. Our results match those of the gradient flow up to discretization errors that depend on the step-size. Section~\ref{sec:expes} eventually provides numerical experiments validating our theoretical findings.
\section{Preliminaries}\label{sec:prelim}
\subsection{General Notations}
For a matrix $\Mv \in \R^{a \times b}$ we denote by $\sigmin(\Mv)$ and $\sigmax(\Mv)$ its smallest and largest non-zero singular values, and by $\kappa(\Mv) = \frac{\sigmax(\Mv)}{\sigmin(\Mv)}$ its condition number. We use as a shorthand $\sigminA$ to express $\sigmin(\fop)$.  We denote by $\norm{\cdot}$ the Euclidian norm of a vector. With a slight abuse of notation $\norm{\cdot}$ will also denote the spectral norm of a matrix. We use the notation $a \gtrsim b$ if there exists a constant $C > 0$ such that $a \geq C b$.
We also define $\xvtau = \gdiptau$ and recall that $\yvtau = \fop\xvtau$. The Jacobian operator of $\gv(\uv,\cdot)$ is denoted $\Jg$. $\Jgtau$ is a shorthand  notation of $\Jg$ evaluated at $\thetavtau$. The local Lipschitz constant of a mapping on a ball of radius $R > 0$ around a point $\zv$ is denoted $\Lip_{\Ball(\zv,R)}(\cdot)$. We omit $R$ in the notation when the Lipschitz constant is global. For a function $f: \R^n \to \R$, we use the notation for the sublevel set $[f < c] = \enscond{\zv \in \R^n}{f(\zv) < c}$ and $[c_1 < f < c_2] = \enscond{\zv \in \R^n}{ c_1 < f(\zv) < c_2}$. We set $\Cphi = \sqrt{\Expect{X\sim\stddistrib}{\phi(X)^2}}$ and $\Cphid = \sqrt{\Expect{X\sim\stddistrib}{\phi'(X)^2}}$.


For some $\Theta \subset \R^p$, we define $\Sigma_\Theta = \enscond{\gv(\uv,\thetav)}{\thetav \in \Theta}$ the set of vectors that the network $\gv(\uv,\cdot)$ can generate for all $\thetav$ in the set of parameters $\Theta$. $\Sigma_\Theta$ can thus be viewed as a parametric manifold. If $\Theta$ is closed (resp. compact), so is $\Sigma_\Theta$. We denote $\dist(\cdot,\Sigma_\Theta)$ the distance to $\Sigma_\Theta$ which is well defined if $\Theta$ is closed and non-empty. For a vector $\xv$, $\xvsigmatheta$ is its projection on $\Sigma_\Theta$, i.e. $\xvsigmatheta \in \Argmin_{\zv \in \Sigma_\Theta} \norm{\xv-\zv}$. We also define $T_{\Sigma_\Theta}(\xv)$ the tangent cone of $\Sigma_\Theta$ at $\xv\in\Sigma_\Theta$. The minimal (conic) singular value of a matrix $\fop \in \R^{m \times n}$ w.r.t. the cone $T_{\Sigma_\Theta}(\xv)$ is then defined as
\[
\lmin(\fop;T_{\Sigma_\Theta}(\xv)) = \inf
\{\norm{\fop \zv}/\norm{\zv}:  \zv \in T_{\Sigma_\Theta}(\xv)\}.
\]
Throughout, $\gdip$ is a feedforward neural network.
\begin{definition}\label{def:nn}
Let $\phi : \R \to \R$ be a component wise activation function. An $L$-layer fully connected neural network is a collection of weight matrices $\pa{\Wv^{(l)}}_{l \in [L]}$ where $\Wv^{(l)} \in \R^{N_l\times N_{l-1}}$ with $N_l \in \N$ the number of neurons on layer $l$. We take $\thetav \in \R^p$, with $p = \sum_{l=1}^L N_{l-1}N_l$, a vector that gathers all these parameters. Then, a neural network parametrized by $\thetav$ is the mapping $\gv: (\uv,\thetav) \in \R^d \times \R^p \mapsto \gv(\uv,\thetav) \in \R^{N_L}$ with $N_L =n$, which is defined recursively as
\begin{align*}
		\begin{cases}
			\gv^{(0)}(\uv,\thetav)&= \uv,
			\\
			\gv^{(l)}(\uv,\thetav)&=\phi\pa{\Wv^{(l)}\gv^{(l-1)}(\uv,\thetav)}, \quad  l=1,\ldots , L-1,
			\\
			\gv(\uv,\thetav)&= \Wv^{(L)} \gv^{(L-1)}(\uv,\thetav).
		\end{cases}
	\end{align*}

\end{definition}
\subsection{KL inequality}
We will work under a general condition of the loss function $\loss$ which includes non-convex ones. We will suppose that $\loss$ verifies a Kurdyka-\L ojasewicz-type (KL for short) inequality~\cite[Theorem~1]{kurdyka_gradients_1998}. 
\begin{definition}[KL inequality]\label{def:KL}
A continuously differentiable function $f:\funspacedef{\R^n}{\R}$ with $\min f=0$ satisfies the KL inequality if there exists $r_0 > 0$ and a strictly increasing function $\psi \in \cC^0([0,r_0[) \cap \cC^1(]0,r_0[)$ with $\psi(0) = 0$ such that
\begin{align}\label{eq:KLpsi}
\psi'(f(\zv))\norm{\nabla f(\zv)} \geq 1, \qforallq \zv \in [f < r_0] . 
\end{align}
We use the shorthand notation $f \in \KLpsi(r_0)$ for a function satisfying this inequality.
\end{definition}

\section{Recovery Guarantees with Gradient Descent} \label{sec:guarantees}

\subsection{Main Assumptions} 
Throughout this paper, we will work under the following standing assumptions:
\vspace*{0.5em}

\begin{mdframed}
    \begin{assumption}\label{ass:l_smooth}
        $\lossy(\cdot) \in \cC^1(\R^m)$ is bounded from below whose gradient is Lipschitz continuous on the bounded sets of~$\R^m$. 
    \end{assumption}
    \begin{assumption}\label{ass:l_kl}
        $\lossy(\cdot) \in\KLpsi(\lossy(\yvz)+\eta)$ for some $\eta > 0$. 
    \end{assumption}
    \begin{assumption}\label{ass:min_l_zero}
        $\min \lossy(\cdot) = 0$.
    \end{assumption}
    \begin{assumption}\label{ass:phi_diff}
        $\phi \in \cC^1(\R)$ and $\exists B > 0$ such that $\sup_{x \in \R}|\phi'(x)| \leq B$ and $\phi'$ is $B$-Lipschitz continuous.
    \end{assumption}
    \begin{assumption}\label{ass:theta_bounded}
        $\thetavseqtau$ is bounded.
    \end{assumption}
    
\end{mdframed}


Notably, these assumptions are not restrictive. \ref{ass:l_smooth} and \ref{ass:phi_diff} ensure that $\nabla_{\thetav}\lossy(\fop\gv(\uv,\cdot))$ is locally Lipschitz w.r.t $\thetav$ and \ref{ass:l_kl} is met by many classical loss functions (MSE, Kullback-Leibler and cross entropy to cite a few). \ref{ass:theta_bounded} is quite mild as we do not require neither convexity nor coercivity of the objective or that $\nabla_{\thetav}\lossy(\fop\gv(\uv,\cdot))$  is globally Lipschitz continuous. Indeed, the latter is widely assumed to ensure that the scheme~\eqref{eq:grad_descent} has the descent property, but this is unrealistic when training neural networks. Our assumption \ref{ass:theta_bounded}, together with \ref{ass:l_smooth} and \ref{ass:phi_diff}, ensure the existence of a constant $L > 0$ such that $\nabla_{\thetav}\lossy(\fop\gv(\uv,\cdot))$ is $L$-Lipschitz continuous w.r.t $\thetav$ on the ball containing $\thetavseqtau$. This avoids an ``egg and chicken'' issue. Indeed, we can show that $\thetavseqtau$ is bounded if \eqref{eq:grad_descent} has a global descent property, but this requires global Lipschitz continuity of $\nabla_{\thetav}\lossy(\fop\gv(\uv,\cdot))$. \ref{ass:theta_bounded} resolves that issue.

\medskip

\subsection{Deterministic Results}




We can now state our recovery theorem for gradient descent. For obvious space limitation, the proof can be found in Appendix~\ref{appendix:proof_main_thm}.
\begin{theorem}\label{thm:main_discr}
Consider a network $\gv(\uv,\cdot)$, a forward operator $\fop$ and a loss $\lossy$ such that our assumptions hold. 
    Let $\thetavseqtau$ be the sequence generated by \eqref{eq:grad_descent}. There exists a constant $L > 0$ such that if $\gamma \in ]0,1/L]$ and if the initialization $\thetavz$ is such that
    \begin{equation}\label{eq:bndR2}
        \sigminjgz > 0 \qandq R' < R
    \end{equation}
    where R' and R obey
    \begin{equation}\label{eq:RandR'2}
        R' = \frac{2\nu_1\psi(\lossyzy)}{\sigminA\sigminjgz} \qandq R = \frac{\sigminjgz}{2\Lip_{\Ball(\thetavz,R)}(\Jg)},
    \end{equation}
    with $\nu_1=\frac{1+\gamma L}{1-\gamma L / 2} \in ]1,4]$, then the following holds:
    \begin{enumerate}[label=(\roman*)]
\item \label{thm:main_y_bounded2} the loss converges to $0$ at the rate
\begin{align}\label{eq:lossrate2}
    \lossytauy &\leq \Psi\inv\left(\xi_{\tau}\right)
\end{align}
with $\Psi$ a primitive of $-(\psi')^2$ and $\xi_{\tau} = \frac{\sigminA^2\sigminjgz^2}{4\nu_2}\tau + \Psi(\lossyzy)$ where $\nu_2 = \frac{(1+\gamma L)^2}{(\gamma - \gamma^2L/2)}$. Moreover, $\thetavseqtau$ converges to a global minimizer $\thetav_{\infty}$ of $\lossy(\fop\gv(\uv,\cdot))$, at the rate
\begin{align}\label{eq:thetarate2}
\norm{\thetavtau - \thetav_{\infty}} &\leq \frac{2\nu_1\psi\left( \Psi\inv(\xi_{\tau}) \right)}{\sigminjgz\sigminA};
\end{align}


\item if $\Argmin(\lossy(\cdot))=\{\yv\}$, $\lossy$ is convex and
\begin{assumption}\label{ass:A_inj}
$\ker{(\fop)} \cap T_{\Sigma'}(\xvc_{\Sigma'}) = \{0\}$ with $\Sigma'\eqdef \Sigma_{\Ball_{R' + \norm{\thetavz}}}$
\end{assumption}
Then,
\begin{align}\label{eq:xrate_discr}
\norm{\xvtau - \xvc} &\leq \frac{\psi\pa{\Psi\inv\pa{\xi_{\tau}}}}{\lmin(\fop;T_{\Sigma'}(\xvcsigma))} \nonumber\\
&+ \pa{1+\frac{\norm{\fop}}{\lmin(\fop;T_{\Sigma'}(\xvcsigma))}}\dist(\xvc,\Sigma') \nonumber\\
&+ \frac{\norm{\veps}}{\lmin(\fop;T_{\Sigma'}(\xvcsigma))}.
\end{align}

\end{enumerate}
\end{theorem}

\subsection{Discussion and Consequences}

We start by discussing the conditions of the theorem. The first condition $R'<R$ ensures that the loss at initialization is ``small enough'' so that the parameters $\thetavz$ lie in the attraction basin of a minimizer.  The condition $\gamma \in ]0,1/L]$ is classical and ensures that the scheme \eqref{eq:grad_descent} has a descent property. In our context, avoiding big steps guarantees that it will not step out of the attraction basin of the minimizer.
Let us now comment on the different claims of the theorem. The first one ensures that the network converges to a zero-loss solution with a rate dictated by the mapping $\Psi\inv$ (which is decreasing) applied to an affine increasing function of $\tau$. The choice of the loss function is generally dictated by a fidelity argument to the forward model \eqref{eq:prob_inv} (e.g., noise $-$log-likelihood). On the other hand, $\Psi$ only depends on the KL desingularizing function of the chosen loss function. Thus, this choice not only influences the theoretical convergence rate but also the condition \eqref{eq:bndR2} required for the theorem to hold. Our second claim shows that gradient descent provides a sequence of network parameters that converges to a global minimum of the loss. On the signal recovery side, our third result gives us a reconstruction error bound that holds under a restricted injectivity constraint. This bound depends on both the expressivity of the network (via the set $\Sigma'$) and decreases with decreasing noise. Observe that a more expressive network, i.e., larger $\Sigma'$, yields a lower $\dist(\xvc,\Sigma')$ but may hinder the restricted injectivity condition \ref{ass:A_inj}. The recovery bound \eqref{eq:xrate_discr} also suggests to use an early stopping strategy to ensure that $\xv(\tau)$ for $\tau$ large enough (whose expression is left to the reader), will lie in a ball around $\xvc$.

Theorem~\ref{thm:main_discr} shows that gradient descent, which is nothing but an Euler forward time discretization of the gradient flow, inherits the recovery properties of latter proved in~\cite{buskulic2023convergenceJournal}. The price to be paid by gradient descent is in the requirement that the sequence of iterates a priori bounded, that $\gamma$ is well-chosen, as well as by the constants $\nu_1$ and $\nu_2$. As $\nu_1 > 1$ and strictly increasing in $\gamma$, condition $R' < R$ in \eqref{eq:bndR2} is strictly worse that its counterpart for gradient flow, and the error bound \eqref{eq:thetarate2} is also strictly larger. On the other hand, the convergence behaviour gets faster with decreasing $\nu_2$. If $\gamma$ gets close to 0, $\nu_1$ will get smaller but $\nu_2$ will grow very large, entailing a very slow convergence. In fact, $\nu_2$ is minimized when $\gamma=\frac{1}{2L}$ in which case $\nu_2=6L$. As such, this choice of the step-size incurs a trade-off between the convergence speed and the requirements of the theorem.

\subsection{Probabilistic Bounds For a Two-Layer DIP Network}\label{sec:dip}

We now study the case of a two-layer network in the DIP setting. We fix the latent input $\uv$ and learn the weights of a network to match the observation $\yv$. We consider the two-layer network defined as
\begin{align}\label{eq:dipntk}
    \gdip = \frac{1}{\sqrt{k}}\Vv\phi(\Wv\uv)
\end{align}
with  $\Vv \in \R^{n \times k}$ and $\Wv \in \R^{k \times d}$. We will assume the following:


\begin{mdframed}
\begin{assumption}\label{ass:u_sphere}
$\uv$ is a uniform vector on $\sph^{d-1}$; 
\end{assumption}
\begin{assumption}\label{ass:w_init}
$\Wv_0$ has iid entries from $\stddistrib$ and $\Cphi, \Cphid < +\infty$; 
\end{assumption}
\begin{assumption}\label{ass:v_init}
$\Vv_0$ is independent from $\Wv_0$ and $\uv$ and has iid columns with identity covariance and $D$-bounded centered entries.
\end{assumption}
\end{mdframed}

Our main result gives a bound on the level of overparametrization which is sufficient for \eqref{eq:bndR2} to hold with high probability. 

\begin{corollary}\label{cor:dip_two_layers_converge_discr}
Suppose that assumptions \ref{ass:l_smooth} and \ref{ass:phi_diff} hold. Let $C$, $C'$ two positive constants that depend only on the activation function and $D$. Let:
\[
\LLz = \max_{\vv \in \Ball\pa{0,C\norm{\fop}\sqrt{n\log(d)}+\sqrt{m}\pa{\norminf{\fop\xvc}+\norminf{\veps}}}} \frac{\norm{\nabla_\vv \lossy(\vv)}}{\norm{\vv-\yv}}.
\]
Consider the one-hidden layer network \eqref{eq:dipntk} where both layers are trained with gradient descent using initialization satisfying \ref{ass:u_sphere} to \ref{ass:v_init} and the architecture parameters obeying
\begin{align*}
k \geq C' \sigminA^{-4} n \psi\biggl(\frac{\LLz}{2}\Bigl(&C\norm{\fop} \sqrt{n\log(d)} \\
&+ \sqrt{m}\pa{\norminf{\fop\xvc} + \norminf{\veps}}\Bigl)^2\biggl)^4 .
\end{align*}
Then \eqref{eq:bndR2} holds with probability at least $1 - 2n^{-1} - d^{-1}$.
\end{corollary}

%

The overparametrization bound is strongly dependent on the chosen loss function via the associated desingularizing function $\psi$. As an example, for the MSE loss, we have $\psi(s) = cs^{1/2}$. In that setting, the dependency on the problem variables becomes $k\gtrsim\kappa(\fop)n^2m$ which indicates naturally that the conditioning of the operator plays an important role in the reconstruction capabilities of the network. Furthermore, the dimension of the signal is more impactful on the bound than the dimension of the observation which seems intuitive as the network tries to reconstruct a vector in $\R^n$ while matching a vector in $\R^m$. This bound matches the one from the gradient flow case as the discretization error is absorbed in the constant. A more in-depth discussion can be found in~\cite{buskulic2023convergenceJournal}.

{\noindent \bf Local Lipschitz constant estimate.~}
Estimating the local Lipschitz constant $L$ of $\nabla_{\thetav}\lossy(\fop\gv(\uv,\cdot))$ is not easy in general. One can still derive crude bounds. To give some guidelines, and make the discussion easier, we here focus on the MSE loss and assume that $\phi$ is uniformly bounded (this conforms to our numerics). In this case, we actually have global $L$-Lipschitz continuity $\nabla_{\thetav}\lossy(\fop\gv(\uv,\cdot))$ and one can show that
{\small
\begin{equation}\label{eq:Lbnd}
L \lesssim \norm{\fop}^2\frac{n}{k} + \norm{\fop}\norm{\yv-\yvz} \sqrt{\frac{n}{k}} ,
\end{equation}
}
where the constant in $\lesssim$ depends only on $B$, $D$ and the bound on $\phi$. As the level of overparametrization increases, the bound on $L$ gets smaller and one can afford taking larger step-size $\gamma$. On the other hand, for the overparametrized regime, $L$ depends essentially on $\norm{\fop}$ and the loss at initialization, which may themselves depend on the dimensions $(m,n)$, indicating that the choice of $\gamma$ may be influenced by $(m,n)$, regardless of the choice of $k$ large enough.

\section{Numerical Experiments}\label{sec:expes}

We trained two-layer networks, equipped with the sigmoid activation, with varying architectures and on problems with varying dimensions. In the first experiments we set $n=5$, and sample iid entries of $\fop$ and $\xvc$ from the standard Gaussian distribution. Our first experiment aims at verifying numerically the level of overparametrization needed for condition~\eqref{eq:bndR2} to hold and compare it with the one that yields convergence to a zero loss. In this experiment, we use the MSE loss and only train the hidden layer of the network as discussed in~\cite{buskulic2023convergenceJournal}. From Figure~\ref{fig:r_r'} we see that our theoretical bound in Corollary~\ref{cor:dip_two_layers_converge_discr} is validated, though it is pessimistic as our analysis is a worst-case one. For instance, for $m=2$, our bound is quite conservative as it requires $k\geq 10^{6.8}$ while in practice, convergence occurs when $k\geq10^{2.5}$.

\begin{figure}[htb!]
    \centering
    \begin{subfigure}[t]{.47\textwidth}
        \centering
        \includegraphics[width=1\linewidth]{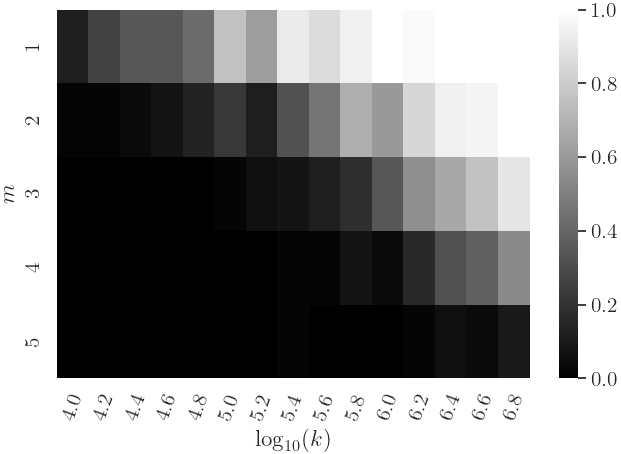}
        \caption{Empirical probability over 50 runs that \eqref{eq:bndR2} is verified as a function of $m$ and $k$.}
        \label{fig:r_r'_th}
    \end{subfigure}%
    \hspace{0.04\textwidth}
    \begin{subfigure}[t]{.47\textwidth}
        \centering
        \includegraphics[width=1\linewidth]{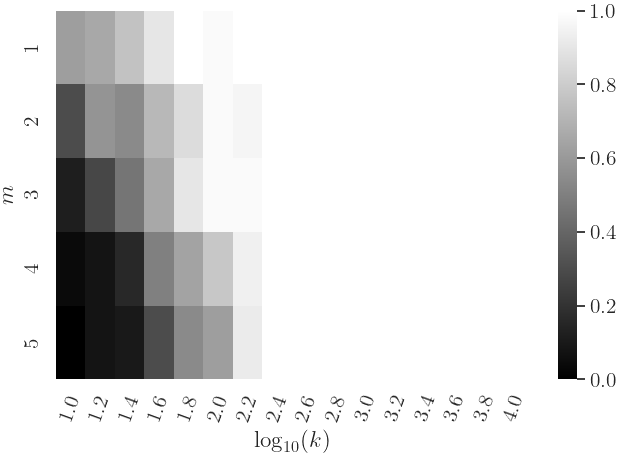}
        \caption{Empirical probability over 50 runs to converge to an optimal solution as a function of $m$ and $k$, i.e., where $\lossy(\yv_{10^5}) \leq 10^{-14}$.}
        \label{fig:r_sigmin}
    \end{subfigure}%
    \caption{Level of overparametrization needed for~\eqref{eq:bndR2} to hold compared to the one required to converge in practice.}
    \label{fig:r_r'}
\end{figure}

Our second experiment depicted in Figure~\ref{fig:plots_tau_n} aims at verifying experimentally the proper choice of $\gamma$ for gradient descent to converge to an optimal solution with different values $n$ and fixed $k=10^4$. For the tested range of $n$, the network is overparameterized and thus $n/k=O(1)$ in \eqref{eq:Lbnd}. One can see from Figure~\ref{fig:plots_tau_n} that there is a clear correlation between the dimension of the problem $n$ and the allowed choice of $\gamma$. Indeed, by concentration of Gaussian matrices, we have for any $\delta > 0$, $\norm{\fop} \leq \sqrt{n}+(1+\delta)\sqrt{m}$ with high probability. Thus, according to \eqref{eq:Lbnd}$, \gamma$ must get smaller as $n$ increases. This is confirmed by another experiment that we do not show here for obvious space limitation reasons, where we have observed that when we fix $n$ and $m$ and vary $k$, there is a threshold effect where, whatever $k$ is chosen, the network training will start to diverge when $\gamma \geq 10^{1.6}$, validating the earlier discussion.


\begin{figure}[t]
        \centering
        \includegraphics[width=0.5\linewidth]{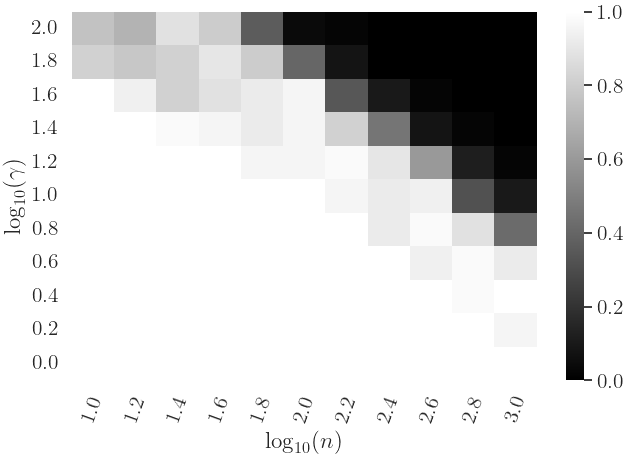}
        \caption{Probability over 50 runs for a network to converge to an optimal solution for various $n$ and $\gamma$.}
        \label{fig:plots_tau_n}
    \end{figure}%
    
    

For our next experiment, we will ilustrate how our models behave on images taken from the tiny ImageNet dataset, that we consider as vectors in $[0,255]^{4096}$. We used $k=10^4$ as it was empirically sufficient to achieve convergence. In Figure~\ref{fig:gaussian_blur}, $\fop$ is a convolution with a Gaussian kernel of standard deviation 1 (thus, $n=m$ here), and tested two scenarios: noiseless and one with low level of an additive zero-mean white Gaussian noise (AWGN) with standard deviation 2.5. In the noiseless case, the original image is perfectly recovered. The presence of noise, even small, entails a degraded reconstruction in the image space, thought the reconstruction is very good in the observation space. This is predicted by our theoretical results due to the fact that the blur operator, while being injective, is very badly conditioned, that is $\sigmin(\fop)\sim 10^{-5}$, hence greatly amplifying the noise in the reconstruction (see \eqref{eq:xrate_discr}). An early stopping strategy is thus necessary to avoid overfitting the noise.


\begin{figure*}[htb!]
    \centering
    \begin{subfigure}[t]{.99\textwidth}
        \centering
        \includegraphics[width=1\linewidth]{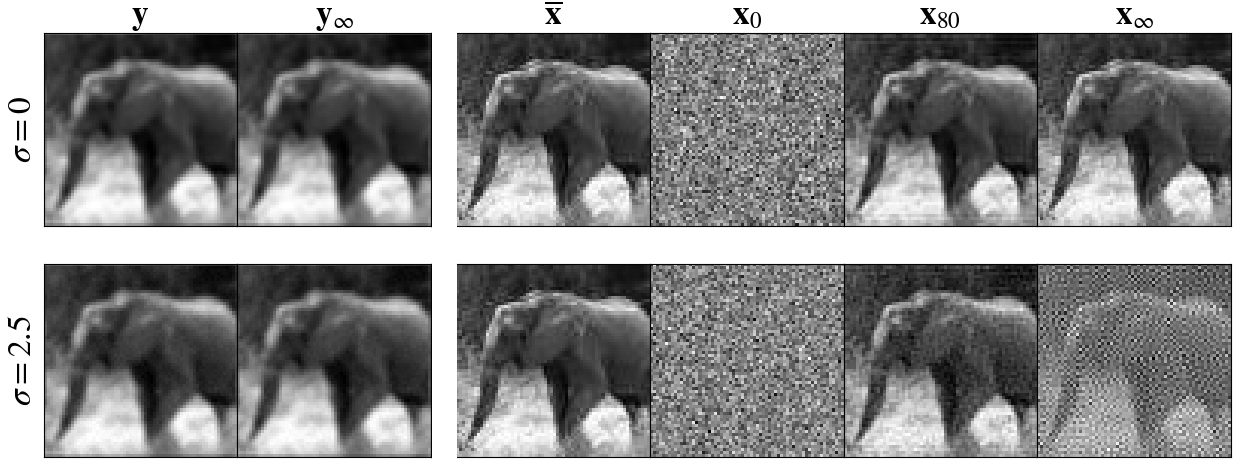}
        \caption{Reconstruction of an image when one uses Gaussian blur with no noise and with low level of noise.}
        \label{fig:gaussian_blur}
    \end{subfigure}%
    
    \begin{subfigure}[t]{.99\textwidth}
        \centering
        \includegraphics[width=1\linewidth]{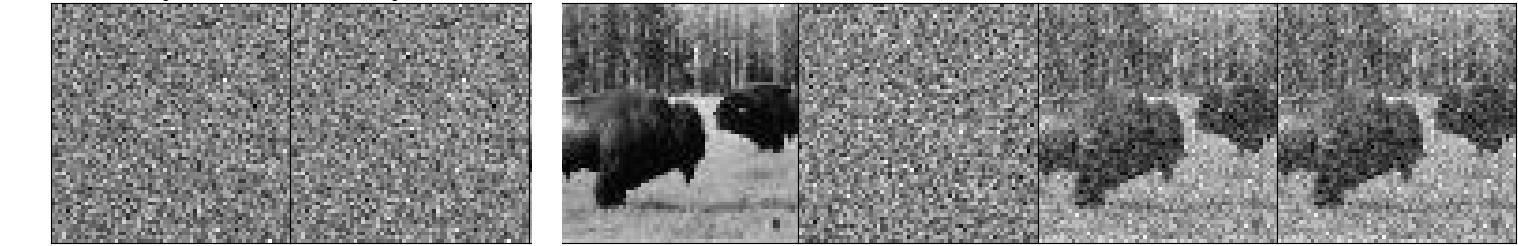}
        \caption{Evolution of the reconstruction through training with a well-conditioned operator and high level of noise.}
        \label{fig:Evolution_dip}
    \end{subfigure}%
    \caption{Deep Inverse Prior applied to image reconstruction.}
    \label{fig:imagerec}
\end{figure*}


In Figure~\ref{fig:Evolution_dip}, we changed $\fop$ to a better conditioned one. For this, we generated two random orthogonal matrices, a diagonal matrix with entries evenly spaced between 1 and 2, and then formed $\fop$ whose SVD is these three matrices. We used an AWGN with standard deviation 50, entailing that the observation is overwhelmed by noise (see left image of Figure~\ref{fig:Evolution_dip}). Figure~\ref{fig:Evolution_dip} also displays the reconstructed images obtained at increasing iterations of gradient descent training. We see that the reconstructed image gets better after a few iterations, before converging to a solution. To confirm these visual results, we report in Figure~\ref{fig:signal_evoluion} the evolution of the reconstruction error $\norm{\xvtau - \xvc}$ vs the iteration counter (the upper-bound predicted by our theorem is also shown in dashed line). We observe that indeed the recovery error decreases until $\tau \approx 10$ and then increases slightly before stabilizing.


\begin{figure}[!htb]
    \centering
    \includegraphics[width=0.5\linewidth]{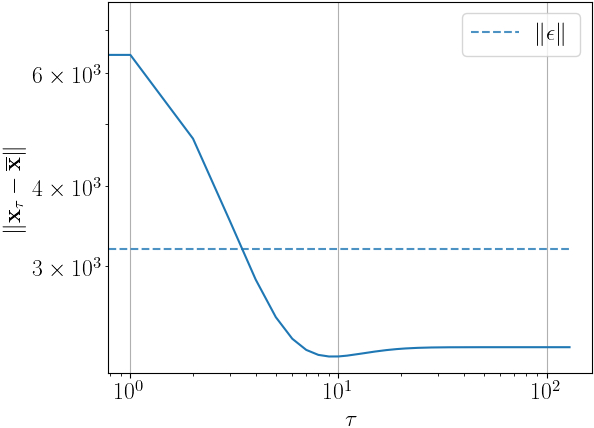}
    \caption{Evolution of the reconstruction error with a well-conditioned operator and high amount of noise.}
    \label{fig:signal_evoluion}
\end{figure}


\section{Conclusion and Future Work}\label{sec:conclu}
We studied the convergence of unsupervised networks in the inverse problem setting when trained through gradient descent. We showed that both convergence and recovery guarantees could be obtained which match the findings of the gradient flow case. Furthermore, we characterized the discretization error incurred which is of a constant factor at most. We also provided an overparametrization bound under which a two-layer DIP network will benefit from these guarantees with high probability. We would like in the future to study the supervised case in the setting where the network is seen as a push-forward measure to generate points on a manifold.

\bibliography{references}

\appendix

\section{Proofs}
\subsection{Proof of Theorem~\ref{thm:main_discr}}\label{appendix:proof_main_thm}
We first present useful lemmas that will help us prove the claims from Theorem~\ref{thm:main_discr}.

\begin{lemma}\label{lemma:descent_trajectory}
    Assume that~\ref{ass:l_smooth} and~\ref{ass:phi_diff} hold. Let $\Omega$ be a bounded convex set of $\R^p$. Then there exists a $L > 0$ such that for any $\thetav,\widetilde{\thetav} \in \Omega$, 
    \[
    |\lossy(\fop \gv(\uv,\widetilde{\thetav}))) - \lossy(\fop \gv(\uv,\thetav))) - \dotprod{\nabla_{\thetav} \lossy(\fop \gv(\uv,\thetav)))}{\widetilde{\thetav}-\thetav}| \leq \frac{L}{2} \norm{\widetilde{\thetav}-\thetav}^2 .
    \]
\end{lemma}
\begin{proof}
   By \ref{ass:l_smooth} and~\ref{ass:phi_diff} , we know that $\nabla_{\thetav} \lossy(\fop \gv(\uv,\cdot)))$ is Lipschitz continuous on bounded sets. Thus, for a convex bounded set $\Omega \subset \R^p$, there exists $L$ such that $\nabla_{\thetav} \lossy(\fop \gv(\uv,\cdot)))$ is $L$-Lipschitz continuous on $\Omega$. Now, take $\thetav,\widetilde{\thetav} \in \Omega$, we have by Taylor formula with integral term that
   \begin{multline*}
          \lossy(\fop \gv(\uv,\widetilde{\thetav}))) - \lossy(\fop \gv(\uv,\thetav))) - \dotprod{\nabla_{\thetav} \lossy(\fop \gv(\uv,\thetav)))}{\widetilde{\thetav}-\thetav} \\
          = \int_{0}^1 \dotprod{\nabla_{\thetav} \lossy(\fop \gv(\uv,\thetav+s(\widetilde{\thetav}-\thetav))))-\nabla_{\thetav} \lossy(\fop \gv(\uv,\thetav)))}{\widetilde{\thetav}-\thetav} ds.
      \end{multline*}
      By convexity, $\thetav+s(\widetilde{\thetav}-\thetav) \in \Omega$ for any $s \in [0,1]$, and $L$-Lipschitz continuity of $\nabla_{\thetav} \lossy(\fop \gv(\uv,\cdot)))$ on $\Omega$ then yields
      \begin{align*}
          &|\lossy(\fop \gv(\uv,\widetilde{\thetav}))) - \lossy(\fop \gv(\uv,\thetav))) - \dotprod{\nabla_{\thetav} \lossy(\fop \gv(\uv,\thetav)))}{\widetilde{\thetav}-\thetav}| \\
          &\leq \int_{0}^1 \norm{\nabla_{\thetav} \lossy(\fop \gv(\uv,\thetav+s(\widetilde{\thetav}-\thetav))))-\nabla_{\thetav} \lossy(\fop \gv(\uv,\thetav)))}\norm{\widetilde{\thetav}-\thetav} ds \\
          &\leq L \norm{\widetilde{\thetav}-\thetav}^2 \int_{0}^1 s ds = \frac{L}{2} \norm{\widetilde{\thetav}-\thetav}^2 . 
      \end{align*}
\end{proof}

\begin{lemma}\label{lemma:l_non_increasing}
    Assume \ref{ass:l_smooth}, \ref{ass:min_l_zero} and \ref{ass:phi_diff} hold and that $\thetavseqtau$ is in a bounded convex set of $\R^p$. Then, there exists a $L>0$ such that if $\gamma\in]0,\frac{1}{L}]$,
    \begin{align}\label{eq:descent_size}
        \lossytauplusy - \lossytauy \leq -\eta\norm{\nabla_{\thetav}\lossytauy}^2,
    \end{align}
    with $\eta = \gamma - \frac{L\gamma^2}{2} \in [0,\frac{1}{2L}]$, ensuring that $(\lossytauy)_{\tau\in\N}$ is non-increasing and that $\thetavseqtau$ converges.
\end{lemma}
\begin{proof}
    We apply \ref{lemma:descent_trajectory} with $\thetavtau$ and $\thetavtauplus$ which gives that for a $L>0$
    \begin{align*}
       &|\lossytauplusy - \lossytauy - \dotprod{\nabla_{\thetav} \lossytauy}{\thetavtauplus-\thetavtau}| \leq \frac{L}{2} \norm{\thetavtauplus-\thetavtau}^2\\
       &\lossytauplusy - \lossytauy \leq \left(\frac{L\gamma^2}{2} - \gamma\right)\norm{\nabla_{\thetav}\lossytauy}^2.
    \end{align*}
    Thus choosing $\gamma\in]0,\frac{1}{L}]$ gives that $(\lossytauy)_{\tau\in\N}$ is non-increasing. Combining this with \ref{ass:min_l_zero} which bounds $\lossy$ by below means that $(\lossytauy)_{\tau\in\N}$ converges. From there we see that
    \begin{align*}
        \sum_{\tau=0}^\infty \norm{\nabla_{\thetav}\lossytauy}^2 \leq \frac{1}{\eta}\sum_{\tau=0}^\infty \lossytauy - \lossytauplusy \leq \frac{1}{\eta}\lossy(\yvz)
    \end{align*}
    which shows that $\left(\norm{\nabla_{\thetav}\lossytauy}\right)_{\tau\in\N}$ is summable, in turn ensuring that $\thetavseqtau$ converges.
\end{proof}


\begin{lemma}\label{lemma:link_params_singvals_discr}
Assume \ref{ass:l_smooth} to \ref{ass:theta_bounded} hold. Recall $R$ and $R'$ from \eqref{eq:RandR'2}. Let $\thetavseqtau$ be the sequence given by~\eqref{eq:grad_descent}.
    \begin{enumerate}[label=(\roman*)]
        \item \label{claim:singvals_bounded_if_params_bounded_discr} If $\thetav \in \Ball(\thetavz,R)$ then
        \begin{align*}
            \sigmin(\jtheta) \geq \sigminjgz/2.
        \end{align*}
        
        \item  \label{claim:params_bounded_if_singvals_bounded_discr} There exists a constant $L\in\R^+$ such that if $\gamma \leq \frac{1}{L}$ and if for all $s\in \{0,\dots,\tau\}$, $\sigminjgs \geq \frac{\sigminjgz}{2}$, then
        \begin{align*}
            \thetavtau \in \Ball(\thetavz,R').
        \end{align*}
        
        \item \label{claim:sigval_bounded_everywhere_discr}
        If $R'<R$, then for all $t \geq 0$, $\sigminjgt \geq \sigminjgz/2$.
\end{enumerate}
\end{lemma}
\begin{proof}
The proof of~\ref{claim:singvals_bounded_if_params_bounded_discr} and~\ref{claim:sigval_bounded_everywhere_discr} are found in~\cite[Lemma 3.10]{buskulic2023convergenceJournal}. We now prove claim~\ref{claim:params_bounded_if_singvals_bounded_discr}. By~\ref{ass:theta_bounded}, $\thetavseqtau$ is bounded and thus in a bounded convex set of $\R^p$, which allows us to use Lemma~\ref{lemma:l_non_increasing}. We thus know that there exist $L>0$ such that when $\gamma \leq 1/L$, \eqref{eq:grad_descent} will ensure that $(\lossy(\yvtau))_{\tau\in\mathbb{N}}$ is non-increasing. This allows us to use the KL property of the loss as stated in~\ref{ass:l_kl} for any step $\tau$. With that in place, we can use the following sequence
\begin{align}\label{eq:psi_lyapu}
    V_\tau = \psi(\lossy(\yvtau)) + \alpha \sum_{i=0}^{\tau  - 1}\norm{\thetav_{i+1} - \thetav_i}
\end{align}
where $\psi$ is the desingularizing function given by the KL property of $\lossy$ and $\alpha\in\R^*_+$. Our goal is to show that for a sufficiently small $\alpha$, this sequence is non-increasing. With that in mind, let us start with the fact that
\begin{align*}
    V_{\tau+1} - V_\tau = \psi(\lossy(\yv_{\tau+1})) - \psi(\lossytauy) + \alpha \norm{\thetavtauplus - \thetavtau}.
\end{align*}
By the mean value theorem, we have
\begin{align}\label{eq:mvt_kl}
    \psi(\lossy(\yv_{\tau+1})) - \psi(\lossytauy) = \psi'(\lossy(\yv_{\tau_\delta}))(\lossy(\yv_{\tau+1}) - \lossytauy),
\end{align}
with $\delta\in[0,1]$. Going back to our sequence and using~\eqref{eq:mvt_kl} and~\eqref{eq:descent_size} we obtain
\begin{align*}
    V_{\tau+1} - V_\tau &\leq  \psi'(\lossy(\yv_{\tau_\delta}))(\lossy(\yv_{\tau+1}) - \lossytauy) + \alpha \norm{\thetavtauplus - \thetavtau}\\
    &\leq - \psi'(\lossy(\yv_{\tau_\delta}))\eta\norm{\nabla_{\thetav}\lossytauy}^2 + \alpha\gamma\norm{\nabla_{\thetav}\lossytauy}\\
    &\leq \norm{\nabla_{\thetav}\lossytauy}\left(\alpha\gamma - \psi'(\lossy(\yv_{\tau_\delta}))\eta\norm{\nabla_{\thetav}\lossytauy}\right).
\end{align*}
Our next step is to bound $\psi'(\lossy(\yv_{\tau_\delta}))$:
\begin{align}\label{eq:bound_psi_delta}
    \psi'(\lossy(\yv_{\tau_\delta})) &\geq \frac{1}{\norm{\nabla_{\yv}\lossy(\yv_{\tau_\delta})}} \nonumber\\
    &\geq \frac{\sigmin(\Jgtau)\sigminA}{\norm{\nabla_{\thetav}\lossy(\yv_{\tau_\delta})}} \nonumber\\
    &\geq \frac{\sigmin(\Jgtau)\sigminA}{\norm{\nabla_{\thetav}\lossytauy} + \norm{\nabla_{\thetav}\lossy(\yv_{\tau_\delta}) - \nabla_{\thetav}\lossytauy}}  \nonumber\\
    &\geq \frac{\sigmin(\Jgtau)\sigminA}{\norm{\nabla_{\thetav}\lossytauy} + L\norm{\thetav_{\tau_\delta} - \thetavtau}} \nonumber\\
    &\geq \frac{\sigmin(\Jgtau)\sigminA}{(1+\gamma L)\norm{\nabla_{\thetav}\lossytauy}}.
\end{align}

Going back to the sequence $(V_\tau)_{\tau\in\mathbb{N}}$ and combining this last result with the fact that for any $s$, $\sigminjgs \geq \frac{\sigminjgz}{2}$, we get
\begin{align*}
    V_{\tau+1} - V_\tau &\leq \norm{\nabla_{\thetav}\lossytauy}\left(\alpha\gamma - \eta\frac{\sigminjgz\sigminA}{2(1+\gamma L)}\right)\\
\end{align*}
This shows that as long as $\alpha \leq \frac{\sigminjgz\sigminA}{2\nu}$, with $\nu=\frac{1+\gamma L}{1-\gamma L / 2} \in [1,4]$, $(V_\tau)_{\tau\in\mathbb{N}}$ is non-increasing and thus that $V_\tau\leq V_0$. This allows us to obtain that
    \begin{align*}
        \norm{\thetav_\tau - \thetavz} \leq \sum_{i=0}^{\tau - 1}\norm{\thetav_{i+1} - \thetav_i} \leq \frac{1}{\alpha}\psi(\lossyzy),
    \end{align*}
    which gives us that $\thetavseqtau \in \Ball(\thetavz,R')$ which is the desired claim. 
    
\end{proof}

\begin{proof}[Proof of Theorem~\ref{thm:main_discr}]
\begin{enumerate}[label=(\roman*)]
\item We use several energy functions to show our results. We first take as our energy function the loss $\loss$. By Lemma~\ref{lemma:link_params_singvals_discr}~\ref{claim:sigval_bounded_everywhere_discr} and~\eqref{eq:bndR2}, we have that $\sigmin(\Jg(\tau))\geq\sigminjgz/2$. In turn, Lemma~\ref{lemma:link_params_singvals_discr}~\ref{claim:params_bounded_if_singvals_bounded_discr} gives us that $\thetavseqtau \in \Ball(\thetavz,R')$. With the boundedness of $\thetav$, we can apply classical descent lemma steps and find using Lemma~\ref{lemma:l_non_increasing} with the condition $\gamma\leq1/L$, that the loss is non-increasing w.r.t $\tau$. Embarking from~\eqref{eq:descent_size} we have
\begin{align*}
    \lossy(\yv_{\tau+1}) - \lossytauy &\leq -\eta\norm{\nabla_{\thetav}\lossytauy}^2\\ 
\end{align*}
with $\eta = \gamma - \frac{L\gamma^2}{2} \in [0,\frac{1}{2L}]$. Let $\Psi$ be a primitive of $-\psi'^2$. Then, by the mean value theorem and that $\Psi$ is strictly decreasing we get
\begin{align*}
    \Psi(\lossytauplusy) - \Psi(\lossytauy) &= \Psi'(\lossy(\yv_{\tau_\delta}))\left(\lossy(\yv_{\tau+1}) - \lossytauy\right)\\
    &= -\psi'(\lossytaudeltay)^2\left(\lossytauplusy - \lossytauy\right)\\
    &\geq \psi'(\lossytaudeltay)^2\eta\norm{\nabla_{\thetav}\lossytauy}^2.
\end{align*}
We can use the bound found in~\eqref{eq:bound_psi_delta} and the bound on $\sigmin(\Jg(\tau))$ to obtain
\begin{align*}
    \Psi(\lossytauplusy) - \Psi(\lossytauy) \geq \eta\frac{\sigminjgz^2\sigminA^2}{4(1+\gamma L)^2}.
\end{align*}
From this, we use once again that $\Psi$ and $\Psi\inv$ are (strictly) decreasing to see that
\begin{align*}
    \Psi(\lossytauy) - \Psi(\lossyzy) &= \sum_{i=0}^{\tau-1}\Psi(\lossy(\yv_{i+1})) - \Psi(\lossy(\yv_{i})) \\
    &\geq \frac{\sigminjgz^2\sigminA^2\eta}{4(1+\gamma L)^2}\tau
\end{align*}
Thus,
\begin{align*}
    \lossytauy &\leq \Psi\inv\left( \tau\frac{\sigminjgz^2\sigminA^2\eta}{4(1+\gamma L)^2} + \Psi(\lossyzy) \right),
\end{align*}
which gives~\eqref{eq:lossrate2}.

We now know that the loss converges to 0. The next question is about the rate of convergence of $\thetavseqtau$. First, thanks to Lemma~\ref{lemma:l_non_increasing} we know that $\thetavseqtau$ converges to say $\thetav_\infty$. Going back to the sequence defined in~\eqref{eq:psi_lyapu} which we know is non-increasing when $\alpha =\frac{\sigminjgz\sigminA}{2\nu}$, we can observe that
\begin{align*}
    V_\infty - V_\tau = \psi(\lossy(\yv_\infty)) - \psi(\lossytauy) + \alpha\sum_{i=\tau}^{\infty} \norm{\thetav_{i+1} - \thetav_i} \leq 0.
\end{align*}
Thus, since the loss converges to 0, we get
\begin{align*}
    \norm{\thetavtau - \thetav_\infty} \leq \sum_{i=\tau}^{\infty} \norm{\thetav_{i+1} - \thetav_i} \leq \frac{1}{\alpha}\psi(\lossytauy)
\end{align*}
which proves~\eqref{eq:thetarate2}.

\item We use the same proof as for~\cite[Theorem 3.2 (ii)]{buskulic2023convergenceJournal} and we adapt it to match the new constant in the loss rate.

By continuity of $\fop$ and $\gv(\uv,\cdot)$ we infer that $(\yv_\tau)_{\tau\in\mathbb{N}}$ also converges to some value $\y_\infty$. The continuity of $\lossy(\cdot)$ and~\eqref{eq:lossrate2} gives us
\begin{align*}
    0 = \lim_{\tau \to +\infty}\lossytauy = \lossy(\yv_\infty)
\end{align*}
and thus $\yv_\infty \in \Argmin(\lossy)$. Since the latter is the singleton $\{\yv\}$ by assumption we conclude.

In order to obtain the early stopping bound, we use~\cite[Theorem~5]{bolte2017error} that links the KL property of $\lossy(\cdot)$ with an error bound. In our case, this reads
\begin{align}\label{eq:klerrbnd}
\dist(\yvtau,\Argmin(\lossy)) = \norm{\yvtau - \yv} \leq \psi(\lossytauy).
\end{align}
It then follows that
\begin{align*}
    \norm{\yvtau - \yvc} &\leq \norm{\yvtau - \yv} + \norm{\yv - \yvc}\\
    &\leq \psi(\lossytauy) + \norm{\veps}\\
    &\leq \psi\left(\Psi\inv\left(\eta\frac{\sigminF^2\sigmin(\Jgz)^2\eta}{4(1+\gamma L)^2}\tau + \Psi(\lossyzy)\right)\right) + \norm{\veps}.
\end{align*}
Using that $\psi$ is increasing and $\Psi$ is decreasing, the first term is bounded by $\norm{\veps}$ for all $\tau \geq \frac{4(1+\gamma L)^2\Psi(\psi\inv(\norm{\veps}))}{\sigminF^2\sigminjgz^2}-\Psi(\lossyzy)$.

    \item By lemma~\ref{lemma:link_params_singvals_discr}, we have that $\forall\tau\in\mathbb{N},\thetavtau \in \Ball(\thetavz,R')$ and thus that $\xvtau \in \Sigma'$. Moreover,~\ref{ass:A_inj} ensures that $\lmin(\fop;T_{\Sigma'}(\xvcsigma)) > 0$ which allows us to get the following chain of inequalities:
    \begin{align*}
        \norm{\xvtau - \xvc} &\leq \norm{\xvtau - \xvcsigma} + \dist(\xvc,\Sigma')\\
        \text{\small \ref{ass:A_inj}}&\leq \frac{1}{\lmin(\fop;T_{\Sigma'}(\xvcsigma))} \norm{\yvtau - \fop\xvcsigma} + \dist(\xvc,\Sigma')\\
        &\leq \frac{1}{\lmin(\fop;T_{\Sigma'}(\xvcsigma))}\left( \norm{\yvtau - \yv} + \norm{\yv - \fop\xvc} + \norm{\fop(\xvc - \xvcsigma)}  \right)+ \dist(\xvc,\Sigma')\\
        \begin{split}
         {\small \eqref{eq:prob_inv},\eqref{eq:lossrate2},\eqref{eq:klerrbnd}}  &\leq \frac{\psi(\Psi\inv(\xi_\tau))}{\lmin(\fop;T_{\Sigma'}(\xvcsigma))}+ \frac{\norm{\epsilon}}{\lmin(\fop;T_{\Sigma'}(\xvcsigma))}\\
           &\qquad  + \left(\frac{\norm{\fop}}{\lmin(\fop;T_{\Sigma'}(\xvcsigma))} + 1\right)\dist(\xvc,\Sigma')
        \end{split}
    \end{align*}
    which conclude the proof of~\eqref{eq:xrate_discr}.
\end{enumerate}
\end{proof}

\subsection{Proof of Corollary~\ref{cor:dip_two_layers_converge_discr}}

\begin{proof}
    The only difference between the $R$ and $R'$ from~\eqref{eq:bndR2} and the ones from~\cite[Theorem~3.2]{buskulic2023convergenceJournal} is that $R'$ in our case is multiplied by a constant factor. This allows us to prove this corollary using the same proof structure as the one for~\cite[Theorem~4.1]{buskulic2023convergenceJournal} that we adapt to take the constant in $R'$ into account and the fact that our operator is linear.
    
    Due to the fact that the assumptions of the original Theorem and this corollary are the same, it means that both~\cite[Lemma~4.9]{buskulic2023convergenceJournal} and~\cite[Lemma~4.10]{buskulic2023convergenceJournal} stay the same. Furthermore, since $R$ also remains identical in both cases, we we can directly get from the original proof that
    \[
    R \geq C_1\pa{\frac{k}{n}}^{1/4}
    \]
    whenever $k \gtrsim n$, $C_1$ being a positive constant that depends only on $B$, $\Cphi$, $\Cphid$ and $D$.

    Concerning $R'$, we first need to adjust the result from~\cite[Lemma~4.9]{buskulic2023convergenceJournal} to our linear operator which gives us that
    \begin{align}\label{eq:init_error_linear}
        \norm{\yv(0) - \yv} \leq C\norm{\fop} \sqrt{n\log(d)} + \sqrt{m}\pa{\norminf{\fop\xvc} + \norminf{\veps}} ,
    \end{align}
    with probability at least $1 - d^{-1}$, where $C$ is a constant that depends only on $B$, $\Cphi$, and $D$. From here we follow the original proof and get that by using the descent lemma in~\cite[Lemma~2.64]{BauschkeBook}
    \[
    \lossyzy \leq \max_{\vv \in [\yv,\yvz]}\frac{\norm{\nabla \lossy(\vv)}}{\norm{\vv-\yv}} \frac{\norm{\yvz-\yv}^2}{2} .
    \] 
    Combining this with~\eqref{eq:init_error_linear} and that
    \[
    [\yv,\yvz] \subset \Ball(0,\norm{\yv}+\norm{\yvz}) 
    \]
    allows us to obtain that with probability at least $1 - d^{-1}$ we have
    \[
    \lossyzy \leq \frac{\LLz}{2}\pa{C\norm{\fop} \sqrt{n\log(d)} + \sqrt{m}\pa{\norminf{\fop\xvc} + \norminf{\veps}}}^2 .
    \] 
    Lastly using the union bound with the fact that $\psi$ is increasing we obtain that~\eqref{eq:bndR2} is reached with probability at least $1-2n\inv - d\inv$ when
    \begin{equation}\label{eq:R'Rbnd}
    \frac{16}{\sigminA\sqrt{\Cphi^2 + \Cphid^2}}\psi\pa{\frac{\LLz}{2}\pa{C\norm{\fop} \sqrt{n\log(d)} + \sqrt{m}\pa{\norminf{\fop\xvc} + \norminf{\veps}}}^2} < C_1\pa{\frac{k}{n}}^{1/4} ,
    \end{equation}
    which leads to the claim of the corollary.
\end{proof}

\end{document}